%% file: main_arxiv.tex
\documentclass[11pt]{article}

\usepackage{chngcntr} 
\usepackage{amsmath}
\usepackage{enumerate}
\usepackage{fancyhdr}
\usepackage{color}
\usepackage{listings}
\usepackage{graphicx}
\usepackage{amssymb}
\usepackage{mathtools}

\usepackage{booktabs}
\usepackage{algorithm}
\usepackage{algorithmic}
\usepackage[round]{natbib}
\usepackage{subfig}
\newtheorem{proof}{Proof}

\definecolor{lightgray}{gray}{0.5}

\usepackage[papersize={8.5in,11in},top=1in,bottom=0.75in,left=0.75in,right=0.75in]{geometry}
\usepackage[colorinlistoftodos]{todonotes}
\usepackage{enumitem}
\usepackage{tablefootnote}

\newtheorem{theorem}{Theorem}[section]
\newtheorem{proposition}[theorem]{Proposition}
\newtheorem{lemma}[theorem]{Lemma}
\newtheorem{corollary}[theorem]{Corollary}
\newtheorem{definition}[theorem]{Definition}

\newtheorem{remark}[theorem]{Remark}

\newcommand{\polylog}{\mathsf{polylog}}

\DeclareMathOperator*{\argmax}{argmax}

\title{\textbf{Breaking the $\sqrt{T}$ Barrier: Instance-Independent Logarithmic Regret in Stochastic Contextual Linear Bandits}}

\author{Avishek Ghosh$^{\dagger}$ and Abishek Sankararaman$^{\ddagger}$  \vspace{2mm} \\
Halıcıoğlu Data Science Institute (HDSI), UC San Diego$^\dagger$ \\
\vspace{1.5mm}
AWS AI, Palo Alto, USA$^\ddagger$  \\
email: a2ghosh$@$ucsd.edu, abisanka$@$amazon.com
}

\begin{document}
\maketitle

\begin{abstract}
We prove an instance independent (poly) logarithmic regret for stochastic contextual bandits with linear payoff. Previously, in \cite{chu2011contextual}, a lower bound of $\mathcal{O}(\sqrt{T})$ is shown for the contextual linear bandit problem with arbitrary (adversarily chosen) contexts. In this paper, we show that stochastic contexts indeed help to reduce the regret from $\sqrt{T}$ to $\polylog(T)$. We propose Low Regret Stochastic Contextual Bandits (\texttt{LR-SCB}), which takes advantage of the stochastic contexts and performs parameter estimation (in $\ell_2$ norm) and regret minimization simultaneously. \texttt{LR-SCB} works in epochs, where the parameter estimation of the previous epoch is used to reduce the regret of the current epoch. The (poly) logarithmic regret of \texttt{LR-SCB} stems from two crucial facts: (a) the application of a norm adaptive algorithm to exploit the parameter estimation and (b) an analysis of the shifted linear contextual bandit algorithm, showing that shifting results in increasing regret. We have also shown experimentally that stochastic contexts indeed incurs a regret that scales with $\polylog(T)$.
\end{abstract}

\input{intro}

\input{setup}

\input{main_algo}
\input{simulations}


\bibliographystyle{abbrvnat}
\bibliography{log_regret}

\clearpage
\appendix
\input{proofs}

\end{document}

%% file: intro.tex
\section{INTRODUCTION}
\label{sec:intro}
Contextual bandits are sequential decision making systems, where a learner is typically equipped with $K$ actions (also called ``arms''). At each round $t \in [T]$\footnote{Throughout the text, for positive integer $r$, the notation $[r]$ refers to the set $\{1,2,\ldots,r\}$. } the learner picks an action in the presence of contextual side information. Algorithms for these class of problems typically employ a decision rule that maps the context information to the action chosen. The goal of the learner is to maximize the reward (or in other words, minimize the regret with respect to the best mapping in the hindsight). Contextual bandit paradigm is typically used in advertisement placement \cite{langford_news}, clinical trials \cite{Tewari2017FromAT} and recommendation systems \cite{agarwal2016making}.

The problem of contextual bandits with linear payoffs has a rich body of existing literature. This framework was introduced by \cite{abe2003reinforcement,auer2002using} and further developed in \cite{langford_news,chu2011contextual}. The framework of linear payoff---although simple, is expressive enough to capture several practical real world problems, as explained in \cite{abe2003reinforcement,langford_news}. In particular,  \cite{chu2011contextual} proposes a learning algorithm based on the UCB based optimistic idea. The resulting algorithm, namely SupLinUCB considers arbitrary contexts (i.e., contexts are generated by an adversary) and obtains a high probability regret of $\mathcal{O}(\sqrt{dT \log^3(KT)})$, where $d$ is the dimension of the contexts. In the same paper, it is shown that if the contexts are adversarially generated, any contextual bandit algorithm with linear payoff will incur $\Omega(\sqrt{dT})$ regret. Moreover, several variants of contextual bandits are also studied, for example, in supervised learning \cite{beygelzimer2011contextual}, balanced exploration \cite{dimakopoulou2019balanced} and in delayed systems \cite{zhou2019learning}.

The contextual bandit paradigm has also been investigated beyond linear rewards. As an instance, \cite{agarwal2012contextual} and \cite{pmlr-v32-agarwalb14} consider the $K$-armed generic contextual bandit system and analyzes a regressor elimination type and projection smoothing based learning algorithms respectively, which attains a regret guarantee of $\Tilde{\mathcal{O}}(\sqrt{KT})$. These algorithms are computationally inefficient and depend on an oracle. Furthermore, \cite{foster2020beyond} converts the generic contextual bandit problem to an online regression problem, and obtains similar regret. Recently, \cite{simchi2021bypassing} proposes a learning algorithm, namely \textsf{FALCON}, that obtains $\Tilde{\mathcal{O}}(\sqrt{KT})$ regret in the presence of an offline regression oracle. Moreover, \cite{zhou2020neural} proposes a neural net based learning for contextual bandits.  

In this paper, we stick to the framework of stochastic contextual bandits with linear payoff, and ask the following
\begin{center}
    \textit{``Can (structured) stochastic contexts help in reducing the regret of linear contextual bandits?''}
\end{center}
It turns out, the answer to this question is an astounding \emph{yes}. In fact, if the stochastic contexts satisfy a few regularity conditions,  it is possible to break the  $\Omega(\sqrt{T})$ regret barrier of \cite{chu2011contextual}, and obtain an instance-independent regret of $\mathcal{O}(\polylog \, T)$. We crucially exploit the stochasticity of the contexts. The regularity conditions we impose  (formally written in equation~\ref{eqn:context})  enable us to do statistical estimation (inference) and regret minimization simultaneously. 

We emphasize that bandits with stochastic contexts are also studied quite extensively for contextual linear bandits; for example \cite{clustering_online} uses it for clustering in multi-agent systems, \cite{chatterji2020osom} uses it for binary model selection between linear and standard multi-armed bandits, \cite{ghosh_adaptive} uses it for model selection and \cite{ghosh2021collaborative} uses it for collaboration and personalization in multi-agent systems. Furthermore, for generic contextual bandit problems beyond linear payoffs, the assumption of stochastic contexts is quite common (see \cite{pmlr-v32-agarwalb14,agarwal2012contextual,simchi2021bypassing}).

In this work, we propose an epoch based learning algorithm, namely Low Regret Stochastic Contextual Bandits (\texttt{LR-SCB}). In Theorem~\ref{thm:main_theorem}, we show that the (instance independent) regret of our proposed algorithm scales as\footnote{We have a worse dependence on the context dimension $d$.} $\mathcal{O}( \polylog(T))$. We leverage the concurrent inference and regret minimization aspect to obtain poly-logarithmic regret. Note that previously, in \cite{clustering_online,chatterji2020osom,ghosh2021collaborative}, this simultaneous estimation and regret minimization condition is used to perform additional tasks (on top of regret minimization) such as \emph{clustering, model selection and personalization}.

In \texttt{LR-SCB}, we break the learning horizon into epochs of increasing length. At each epoch, we simultaneously minimize regret and form an estimate of the underlying parameter. Let us assume the underlying parameter for the linear contextual bandit is $\theta^*$. In the first epoch, we play the standard contextual bandit algorithm, OFUL of \cite{chatterji2020osom}\footnote{In fact, we play a variation of the OFUL algorithm, see Section~\ref{sec:main_algo}. For completeness, we reproduce this in Algorithm~\ref{algo:oful}.} with stochastic contexts and learn an estimate $\widehat{\theta}$ of $\theta^*$. Subsequently, in the next epoch, we modify the reward of the learning algorithm in a specific way, such that underlying parameter we need to learn is $\theta^* - \widehat{\theta}$. Hence, the sifted parameter will learn will have a small norm, i,e., $\|\theta^* - \widehat{\theta}\|$ is small, since $\widehat{\theta}$ is an estimate of $\theta^*$. In order to exploit this, we use the norm adaptive algorithm, \texttt{ALB-norm} of \cite{ghosh_adaptive}, which gives regret proportional to the parameter norm. Note that, owing to the proper shift, the norm of the shifted parameter is small, which in turn results in a small regret. We keep on doing this over multiple epochs, and shift the underlying parameter accordingly. With an appropriate choice of epoch lengths, it turns out that this phase based algorithm attains a regret of $\mathcal{O}( \polylog(T))$.


\subsection{Our Contributions:}
\subsubsection{Algorithmic}
We propose an epoch based learning algorithm for stochastic contextual bandits. Our algorithm, \texttt{LR-SCB} introduces proper shifts to the underlying unknown parameters, and uses a norm adaptive algorithm, \texttt{ALB-norm} repeatedly over epochs.  We obtain an instance independent  $\polylog(T)$ regret for the stochastic contextual linear bandit, thus breaking the $\sqrt{T}$ barrier shown in \cite{chu2011contextual}. We show that stochastic contexts indeed help in reducing the regret. To the best of our knowledge, this is the first work to show a (poly) logarithmic instance independent regret for stochastic contextual bandits.
\subsubsection{Technical novelty}
A key technical challenge we encounter is the characterization of \texttt{ALB-norm} under shifts. We argue in Appendix~\ref{app:main} that it is sufficient to understand the behavior of the shifted OFUL system, and in Section~\ref{sec:shift-oful_main} as well as in Appendix~\ref{sec:shift-oful}, we rigorously analyze the shifted OFUL (which might also be of independent interest). For this, we derive an anti-concentration property for the contexts, and in conjunction with independence, we show that OFUL is indeed robust to shifts, and shifting can only increase the regret.

Furthermore, we also require \texttt{ALB-norm} to yield parameter estimation guarantee, similar to OFUL, and in Appendix~\ref{sec:alb_norm_est}, we show that indeed, \texttt{ALB-norm} outputs the required guarantees. 
\subsubsection{Experiments}
We validate our theoretical findings via experiments. In particular for different context dimension, we characterize the regret of \texttt{LR-SCB} with respect to $\log T$, and compare it with OFUL as a baseline. We observe that \texttt{LR-SCB} outperforms OFUL in terms of regret. Furthermore, to understand the regret scaling of \texttt{LR-SCB} better, we plot log regret with respect to $\log \log T$, and obtain a straight line with slope around $2$. This implies that the regret of \texttt{LR-SCB} is indeed $\polylog(T)$, which confirms our theoretical result.

\section{RELATED WORK}
\paragraph{Contextual Bandits:} The literature on contextual bandits is quite rich, starting from \cite{auer2002using,abe2003reinforcement}. Around 2010, with the motivation of recommendation, the study of contextual bandits got some momentum with seminal papers like \cite{langford_news,chu2011contextual}. Most of these papers assume arbitrary, adversarially generated contexts and obtain regret rates of $\mathcal{O}(\sqrt{T})$. Furthermore, several variants of contextual bandits is studied in the literature, for example, in delayed systems \cite{zhou2019learning} and in supervised learning.

Apart from this linear contextual bandits, there has been a significant effort to understand the generic contextual bandits \cite{agarwal2012contextual,agarwal2016making}. Most of these algorithms are non-implementable and very recently \cite{foster2020beyond,simchi2021bypassing} proposes a reduction of the generic contextual bandit problem to an online and offline regression respectively. Very recently, stochastic contexts are used in linear contextual bandits, for example \cite{chatterji2020osom,clustering_online}. The regret guarantee for these algorithms also scale with $\mathcal{O}(\sqrt{T})$. On the other hand, in this work we exploit the stochastic contexts to simultaneously estimate and minimize regret and as a result, we obtain a regret of $\polylog(T)$, thus breaking the $\sqrt{T}$ barrier. 

\paragraph{Adaptive Bandit Algorithms:} As explained in Section~\ref{sec:intro}, the use of an adaptive algorithm that exploits the small norm enables our learning algorithm to obtain logarithmic regret. Adaptive algorithms in bandits have gained a lot of interest in the recent years, for example in \cite{ghosh_adaptive}, the authors define parameter norm and sparsity as complexity parameters for stochastic linear bandit and adapt to those without any apriori knowledge. \cite{foster2019model} also adapts to the sparsity in a linear bandit problem, whereas \cite{pacchiano2020regret} uses the corrall framework of \cite{agarwal2017corralling} to obtain adaptive algorithms for bandits and reinforcement learning. In the corralling framework, the base algorithms are treated as bandit arms, and a learning algorithm is played to choose the correct model. Very recently, the adaptation question is also addressed for generic contextual bandits \cite{krishnamurthy2021optimal,ghosh2021modelone}. Apart from this, in reinforcement learning, a few recent works have started inquiring the question of adaptation, for example \cite{lee2021online} in the framework of function approximation and \cite{ghosh2021model} for generic (but separable) reinforcement learning.

%% file: setup.tex
\section{PROBLEM SETUP}
\label{sec:setup}

We consider the setup of stochastic contextual bandit with linear payoffs \cite{chu2011contextual,chatterji2020osom}. At the beginning of each round $t \in [T]$, the learner chooses one of the $K$ available arms, and gets a reward. To help the learner make the choice of the arm, at each round, the learner is handed $K$ context vectors, $d$ dimensional each, denoted by $\beta_t = [\beta_{1,t},\ldots,\beta_{K,t}]\in \real^{d \times K}$. When the learner chooses arm $i$, the reward obtained is given by $\langle \beta_{i,t}, \theta^* \rangle + \xi_t$, where $\theta^*$ is the $d$-dimensional unknown parameter, with $\|\theta^*\| \leq 1$, and $\{\xi_t\}_{t=1}^T$ denote the noise.

\textbf{Stochastic Assumptions}: We assume that the contexts are stochastic, following the framework of \cite{chatterji2020osom,ghosh_adaptive}. We denote the sigma algebra generated by all noise random variables upto and including time $t-1$ by $\mathcal{F}_{t-1}$. Moreover, by $\mathbb{E}_{t-1}(.)$ and $\mathbb{V}_{t-1}(.)$, we denote the as the conditional expectation and conditional variance operators respectively with respect to $\mathcal{F}_{t-1}$. We further assume that the noise parameter, $(\xi_t)_{t \geq 1}$ are conditionally sub-Gaussian noise with known parameter $\sigma$, conditioned on all the arm choices and realized rewards in the system upto and including time $t-1$, and without loss of generality, let $\sigma=1$ throughout.

The contexts $\{\beta_t\}_{t=1}^T$ are assumed to be bounded---in particular, we let the contexts be drawn from $[-c/\sqrt{d},c/\sqrt{d}]^{\otimes d}$, where $c$ is a universal constant and the $1/\sqrt{d}$ scaling is without loss of generality, so that the norm of the contexts are $\mathcal{O}(1)$.  Moreover, the contexts $\beta_{i,t}$ are assumed to be drawn independent of the past and $\{\beta_{j,t}\}_{j \neq i}$, from a distribution satisfying 
\begin{align}
\label{eqn:context}
    \mathbb{E}_{t-1}[\beta_{i,t}] = 0 \qquad \mathbb{E}_{t-1} [\beta_{i,t} \, \beta_{i,t}^\top] \succeq \rho_{\min} I.
\end{align}
Furthermore, for any fixed $z \in \mathbb{R}^d$, with unity norm, the random variable $(z^\top \beta_{i,t})^2$ is conditionally sub-Gaussian, for all $i$, with $\mathbb{V}_{t-1}[(z^\top \beta_{i,t})^2)] \leq 4 \rho_{\min}$.
This means that the conditional mean of the covariance matrix is zero and the conditional covariance matrix is positive definite with minimum eigenvalue at least $\rho_{\min}$. Furthermore, the conditional variance bound assumption is for technical reasons and is crucially required to apply \eqref{eqn:context} for contexts of (random) bandit arms selected by our learning algorithm (see Lemma 1 of \cite{clustering_online}).

Note this this above set of assumptions on context vectors is not new and the exact set of assumptions were used in \cite{gentile2017context,chatterji2020osom,ghosh2021collaborative,ghosh_adaptive}\footnote{The conditional variance assumption is implicitly used in \citep{chatterji2020osom} without explicit statement.}. In \cite{gentile2017context}, the authors introduced the above-mentioned set of assumptions and use them for parametric inference on top of regret minimization for online clustering problem with bandit information. \cite{chatterji2020osom} uses the same context assumptions for binary model selection between simple multi-armed and contextual linear bandits. Furthermore, \cite{ghosh_adaptive} uses the identical  assumptions to obtain an adaptive problem complexity adaptive regret guarantees for linear bandits and \cite{ghosh2021collaborative} uses these assumptions to ensure personalization for multi-agent linear bandits. Apart from the above mentioned papers, \citep{foster2019model} uses similar assumptions for stochastic linear bandits and \citep{ghosh2021model} uses it for model selection in Reinforcement learning problems with function approximation. In all of the above papers, the authors need parametric inference in conjunction with regret minimization, which is a harder task. If the stochastic contexts are structured, these two tasks can be performed simultaneously. It turns out that the above-mentioned set of assumptions are sufficient to ensure this.

\textbf{Example:} Although we present here the technical conditions needed on contexts, this include simple examples as well. As an instance, it includes the simple setting where the contexts evolve according to a random process independent of the actions and rewards from the learning algorithm. Hence, any zero mean (full rank) iid random variables drawn from a (coordinate-wise) bounded space, generated exogenous to the actions of the agents can be taken as stochastic contexts. As an example, random vectors drawn in an i.i.d manner across rounds from $\mathsf{Unif}[-c_0/\sqrt{d},c_0/\sqrt{d}]^{\otimes d}$ for a constant $c_0$. For this we have $\rho_{\min} = c_1/d$, where $c_1$ is a constant. In Section~\ref{sec:result}, we take this uniform distribution as a special case and completely characterize its perfromance.

Note that the above-mentioned framework of generating contexts are quite standard in the generic contextual bandit literature\cite{agarwal2012contextual,pmlr-v32-agarwalb14} as well, where at each round nature picks a context sampled i.i.d in each round from a fixed and known distribution.

\textbf{Performance Metric:} At time $t$, we denote $B_{t} \in [K]$ as the arm played by the agent. We want to compete with the optimal arm. Since we do not know $\theta^*$, we are bound to incur some error characterized by an equivalent regret term. The regret, over a time horizon of $T$ is given by
\begin{align}
\label{eqn:regret_def}
   R_i(T) = \sum_{t=1}^T  \max_{j \in [K]} \langle  \beta_{j,t}, \theta^* \rangle - \langle \beta_{B_{t},t} , \theta^* \rangle  
\end{align}

%% file: main_algo.tex
\section{Low Regret Stochastic Contextual Bandits (\texttt{LR-SCB})}
\label{sec:main_algo}

Throughout this paper, we refer OFUL as the optimistic learning algorithm of \cite{oful} for linear bandits. In fact \cite{chatterji2020osom} uses this in the finite armed contextual framework, and we use a variation of their OFUL algorithm, without arm biases. For completeness, we reproduce this in Algorithm~\ref{algo:oful}. We use OFUL as a black box in Algorithm~\ref{algo:oful}. 

We now present the algorithm for the stochastic contextual bandit. We divide the learning horizon into epochs of length $T_1,T_2,\ldots,T_N$, where $N$ is the number of epochs. In the first phase $T_1$, we aim to minimize regret and estimate the parameter $\theta^*$ simultaneously for $T_1$ rounds. At the end of this phase, we obtain an estimate $\widehat{\theta}_{T_1}$, of $\theta^*$.

\begin{algorithm}[t!]
  \caption{Low Regret Stochastic Contextual Bandits (\texttt{LR-SCB})}
  \begin{algorithmic}[1]
 \STATE  \textbf{Input:} Horizon $T$, Initial epoch length $T_1$ \\ \vspace{1mm}
 \textbf{First phase,} $i=1$: \\
 \vspace{1mm}
 \STATE Initialize a single instance of OFUL($\delta$) (see Algorithm~\ref{algo:oful})
 \FOR {times $t \in \{1,\cdots,T_1\}$}
 \STATE Play the action given by the common OFUL
 \STATE Update OFUL's state by the observed rewards similar to Algorithm~\ref{algo:oful}
 \ENDFOR
 \STATE Let $ \mathsf{est} \gets \widehat{\theta}_{T_1} $; the parameter estimate of Common OFUL at the end of phase $1$\\ \vspace{1mm}
 \textbf{Subsequent Phases:} \\ \vspace{1mm}
 \FOR {phase $i \in \{2,\ldots,N\}$}
 \STATE $\delta_i \gets \frac{\delta}{2^{i-1}}$
 \STATE $T_i = T_1 (\log T)^{i-1}$
 \STATE Initialize one (modified) ALB-Norm($\delta$) (see Algorithm~\ref{algo:alb_norm} of Appendix) instance per agent
 \FOR {times $t \in \{T_i+1,\ldots,T_{i+1}\}$}
 \STATE Play arm by ALB-Norm (denoted as $\beta_{B_t,t}$) and receive reward $y_t$
 \STATE Every agent updates their ALB-Norm state with corrected reward
 \begin{align*}
     \Tilde{y}_t = y_t -\langle \beta_{B_t,t}, \mathsf{est} \rangle
 \end{align*}
 \ENDFOR
 \STATE $\widehat{\theta}_{T_i}$: parameter estimate after $i$-th epoch
 \STATE $\mathsf{est} \gets \mathsf{est} + \widehat{\theta}_{T_i}$
 \ENDFOR
  \end{algorithmic}
  \label{algo:main_algo}
\end{algorithm}

Subsequently, in the second phase, which lasts for $T_2$ rounds, our goal is to utilize the estimate $\widehat{\theta}_{T_1}$. Here, we aim to learn the parameter $\theta^* - \widehat{\theta}_{T_1}$. Note that, the norm of $\theta^* - \widehat{\theta}_{T_1}$ is small since we spend the previous epoch to learn $\theta^*$. Hence, in this epoch, instead of using the OFUL algorithm, we use an adaptive algorithm that exploits the small norm. In particular, we use a modified version (reproduced in Algorithm~\ref{algo:alb_norm}) of the Adaptive Linear Bandits-norm (\texttt{ALB-norm}) of \cite{ghosh_adaptive}, that exploits the small norm of $\theta^* - \widehat{\theta}_{T_1}$ to obtain a reduced regret, which depends linearly on $\|\theta^* - \widehat{\theta}_{T_1}\|$. As seen in Algorithm~\ref{algo:main_algo}, the learning of $\theta^* - \widehat{\theta}_{T_1}$ is achieved by shifting the reward by the inner product of the estimate $\widehat{\theta}_{T_1}$. By exploiting the anti-concentration of measure along with some standard results from optimization, we show, in Section~\ref{sec:shift-oful_main} as well as in Appendix~\ref{sec:shift-oful} that the regret of the shifted system is worse than the regret of the original system (in high probability)\footnote{This is intuitive since, otherwise one can find \emph{appropriate shifts} to reduce the regret of OFUL, which contradicts the optimality of OFUL.}.

We now continue the above-mentioned estimation procedure in the third epoch as well, which lasts for $T_3$ rounds. Here, we exploit the fact that at the end of the second epoch, we obtain $\widehat{\theta}_{T_2}$, which is an estimate of $\theta^* - \widehat{\theta}_{T_1}$. In Appendix~\ref{sec:alb_norm_est}, we show that similar to the OFUL algorithm, \texttt{ALB-norm} also constructs an estimate of the parameter under consideration. Basically, \texttt{ALB-norm} is equivalent to playing the OFUL algorithm in successive epochs with norm refinements. Using the fact that $\|\theta^* - \widehat{\theta}_{T_1} -\widehat{\theta}_{T_2} \| $ is small, we again use the norm adaptive algorithm \texttt{ALB-norm} to obtain smaller regret. Hence, the regret in this phase is proportional to $\|\theta^* - \widehat{\theta}_{T_1} -\widehat{\theta}_{T_2} \| $.

So, this successive estimation procedure continues upto the $N$-th epoch. At each epoch, we shift the reward by an inner product obtained of the estimate obtained from the previous round. The algorithm is detailed in Algorithm~\ref{algo:main_algo}. Note that in the above algorithm, we use the estimate obtained in the previous epoch and judiciously use a norm adaptive (which adapts to the norm of the problem) algorithm. By judiciously choosing the time epochs, we show that the overall regret of \texttt{LR-SCB} can be reduced to $\mathcal{O}(\polylog T)$.

\begin{algorithm}[t!]
  \caption{OFUL of \cite{chatterji2020osom}}
  \begin{algorithmic}[1]
  \STATE  \textbf{Input:} Parameters $b$, $\delta >0$, number of rounds $\Tilde{T}$
 \FOR{$t = 1,2, \ldots, \Tilde{T} $}
 \STATE Select the best arm estimate as 
 \begin{align*}
      j_t = \mathrm{argmax}_{i\in [K]} \left[ \max_{\theta \in \mathcal{C}_{t-1}} \{  \langle \alpha_{i,t}, \theta \rangle \} \right],
 \end{align*}
where $\mathcal{C}_t $ is the confidence set with radius $\frac{b + \sqrt{d}}{\rho_{\min} \sqrt{t}} \log(K\Tilde{T}/\delta)$
 \STATE Play arm $j_t$, and update $\mathcal{C}_{t}$
 \ENDFOR
  \end{algorithmic}
  \label{algo:oful}
\end{algorithm}
\begin{algorithm}[t!]
  \caption{Adaptive Linear Bandit (norm)--{{\ttfamily ALB-Norm}} of \cite{ghosh_adaptive}}
  \begin{algorithmic}[1]
 \STATE  \textbf{Input:} The initial exploration period $\tau_1$, intial phase length $T_1 := \lceil \sqrt{T} \rceil$,  $\delta_1 > 0$, $\delta_s > 0$.
 \STATE Select an arm at random, sample $2\tau$ rewards
 \STATE Obtain initial estimate ($b_1$) of $\|\theta^*\|$ according to Section $3.3$ of \citep{ghosh_adaptive}.
  \FOR{ epochs $i=1,2 \ldots, N $}
  \STATE Play OFUL (Algorithm~\ref{algo:oful}) with slack $\delta_i$ and norm estimate $b_i$ until the end of epoch $i$ (denoted by $\mathcal{E}_i$)
  \STATE At $t=\mathcal{E}_i$, refine estimate of $\|\theta^*\|$ as, 
  \begin{align*}
      b_{i+1} = \max_{\theta \in \mathcal{C}_{\mathcal{E}_i}} \|\theta\|
  \end{align*}
  \STATE Set $T_{i+1} = 2 T_{i}$
  \STATE $\delta_{i+1} = \frac{\delta_i}{2}$.
    \ENDFOR
  \end{algorithmic}
  \label{algo:alb_norm}
\end{algorithm}

\section{Regret Guarantee for \texttt{LR-SCB}}
\label{sec:result}
In this section, we provide the regret guarantee of \texttt{LR-SCB}. We stick to the notation of Section~\ref{sec:setup}. Moreover,  we select the time epochs in the following manner: $T_i = T_1 (\log T)^{i-1}$. With this choice, the number of epochs is given by, $N = \mathcal{O}\left( \frac{\log(T/T_1)}{\log \log T} \right)$. To ease notation, let us define
\begin{align*}
    \mathsf{\Lambda} = \left( \frac{1}{(\log \log T)} \log \left(\frac{\rho_{\min}^2 T}{d^2 \log^4(KT/\delta) \log(dT/\delta)}\right) \right)
\end{align*}  
and,
\begin{align*}
    \mathfrak{T} &= \log^3 \left(  \frac{K d^2 (\log T) (\log^4 KT/\delta) (\log dT/\delta)}{\rho_{\min}^2 \, \delta} \right) \\
    & \qquad \times \log^2 \left(  \frac{d^3 (\log T) (\log^4 KT/\delta) (\log dT/\delta)}{\rho_{\min}^2 \, \delta} \right)
\end{align*}
We have the following theorem.
\begin{theorem}
\label{thm:main_theorem}
Playing Algorithm~\ref{algo:main_algo} with initial phase length $T_1$ time and probability slack $\delta>0$, where
\begin{align*}
    T_1 = C_1 \, \frac{d^2}{\rho_{\min}^2} \log^4(KT/\delta) \log(dT/\delta) \,\,\,\, \text{and}
\end{align*}
\vspace{-2mm}
\begin{align*}
    d \geq C_1 \frac{\log T}{\log \log T} \log (K^2/\delta).
\end{align*} 
Then the regret of the player for a horizon of $T$ satisfies
\begin{align*}
    R(T) &\leq C_2  \left[ \left( \frac{d}{\rho_{\min}} \right)^{3/2} \,\mathsf{\Lambda}^5 \,\, \mathfrak{T} \,\, \sqrt{\log T} \right] \\
    & = \mathcal{O}\left ( \left( \frac{d}{\rho_{\min}} \right)^{3/2} \,\, \polylog(T,K,d, \delta) \right)
\end{align*}
with probability at least $1-c\delta$, where $c, C, C_1, C_2$ are universal constants.
\end{theorem}
The proof is deferred to the Appendix. We make the following remarks:
\begin{remark}
The above theorem shows that the (instance independent) regret of stochastic contextual bandits is $\polylog(T)$. This is a huge improvement over the $\sqrt{T}$ regret presented in \cite{chu2011contextual,langford_news,chatterji2020osom}. So, the stochastic contexts indeed help in regret reduction.
\end{remark}
\begin{remark}
Note that the dependence on dimension $d$ is worse in \texttt{LR-SCB} compared to SupLinUCB of \cite{chu2011contextual} ($\mathcal{O}\left(\left( \frac{d}{\rho_{\min}} \right)^{3/2} \right)$ vs. $\mathcal{O}(\sqrt{d})$). Furthermore, one needs $d \geq \log(K^2)$ for the anti-concentration of the contexts to kick in, which was crucial in the analysis of the shifted OFUL. 
\end{remark}
\begin{remark}
We require the initial length $T_1 = \Tilde{\mathcal{O}}(d^2/\rho_{\min}^2)$ for the norm adaptive algorithm, \texttt{ALB-norm} to work (see \cite{ghosh_adaptive}).
\end{remark}

\subsubsection{Special Case---Contexts are drawn from Uniform Distribution}
Here we assume the contexts come from $\mathsf{Unif}[-c_0/\sqrt{d},c_0/\sqrt{d}]^{\otimes d}$ for a constant $c_0$. For this we have $\rho_{\min} = c_1/d$, and hence the following result.

\begin{corollary}
Suppose the initial phase length $T_1 = \Tilde{\mathcal{O}}(d^4)$ and $d \geq C_1 \frac{\log T}{\log \log T} \log (K^2/\delta)$. Playing Algorithm~\ref{algo:main_algo} for $T$ times incur a regret of
\begin{align*}
    R(T) \leq \mathcal{O} \left( d^3 \,\, \polylog(T,K,d, \delta) \right),
\end{align*}
with probability at least $1-\delta$.
\end{corollary}

\subsection{Proof Sketch}
\label{sec:sketch}
We now present a brief proof sketch of Theorem~\ref{thm:main_theorem}. The full proof is deferred to Appendix~\ref{app:main}. For simplicity and the clarity of exposition, we only focus on the dependence on time horizon $T$. We break the learning horizon in epochs of lengths $T_1,T_2,\ldots,T_N$.

\textit{Regret in Epoch 1:}
In the first epoch, we play the OFUL algorithm (Algorithm~\ref{algo:oful}). Hence, for \cite{chatterji2020osom}, we  incur a regret of $\mathcal{O}(\sqrt{T_1})$. 

\textit{Regret in Epoch 2:}
In the second epoch, we use the parameter estimate learned in the first epoch and accordingly modify the reward functions. Hence, the underlying parameter in second epoch is the shifted parameter. We leverage the analysis of a shifted OFUL to handle this. Moreover, note that since we are estimating $\theta^*$ in the first epoch, from \cite{chatterji2020osom}, we have
\begin{align*}
    \|\widehat{\theta}_{T_1} - \theta^* \| \leq \mathcal{O}(1/\sqrt{T_1})
\end{align*}
In order to exploit the fact that the norm of the shifted parameter is small, we use a norm-adaptive algorithm, namely \texttt{ALB-norm}, in this round, whose regret is given by
\begin{align*}
    \text{Reg}_{\text{epoch 2}} = \mathcal{O} (\|\widehat{\theta}_{T_1} - \theta^* \|) \sqrt{\frac{1}{T_2}} = \mathcal{O}(\sqrt{\frac{T_2}{T_1}})
\end{align*}

\textit{Regret in Subsequent Epochs:} We continue to shift the parameter by the estimate learnt from the previous epoch. For Epoch 3, we learn $\widehat{\theta}_{T_2}$, which is an estimate of the parameter $\theta^* - \widehat{\theta}_{T_1}$. Using the same \texttt{ALB-Norm}, the regret here is
\begin{align*}
    \text{Reg}_{\text{epoch 3}} = \mathcal{O} (\|\widehat{\theta}_{T_2} - (\theta^* - \widehat{\theta}_{T_1}) \|) \sqrt{\frac{1}{T_3}} = \mathcal{O}(\sqrt{\frac{T_3}{T_2}})
\end{align*}

\textit{Total Regret:} Combining the above expressions, the total regret is given by
 \begin{align*}
     R(T) \leq \mathcal{O}\left( \sqrt{\frac{1}{T_1}} + \sum_{i=1}^N \sqrt{\frac{T_i}{T_{i-1}}}  \right)
 \end{align*}
 
 \textit{Choice of $T_i$:}
 We choose aggressively increasing epoch lengths. This is because, we get to exploit the estimation performance of previous epoch to the new one, and get low regret owing to norm adaptive algorithms. We select $T_i = T_1 (\log T)^{i-1}$, and as a result, the total number of epochs is $ N = \mathcal{O} \left( \frac{\log(T/T_1)}{\log \log T} \right)$.
 
 \textit{Choice of $T_1$:} We use the \texttt{ALB-norm} algorithm of \cite{ghosh_adaptive}, which imposes a condition on $T_1$. It turns out (showed formally in Appendix~\ref{app:main}) we require $T_1 \geq \Tilde{\mathcal{O}}(d^2/\rho_{\min}^2)$. Hence, with the above choice of $T_1$ and combining the regret in different epochs, we obtain
\begin{align*}
    R(T) \leq \mathcal{O}\left( \polylog (T) \right),
\end{align*}
which proves the theorem.

\begin{figure*}[t!]
\label{fig:synthetic}
\centering
    \subfloat[][$d=20, K=20$]{
    \includegraphics[width=0.33\linewidth]{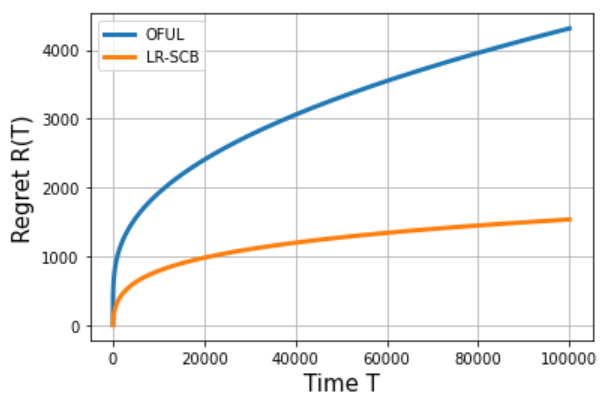}}
    \subfloat[][$d=25,K=20$]{
    \includegraphics[width=0.33\linewidth]{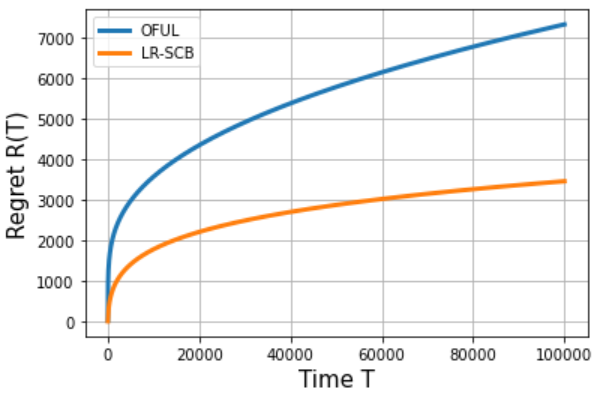}}
    \subfloat[][$d=30, K=20$]{
    \includegraphics[width=0.33\linewidth]{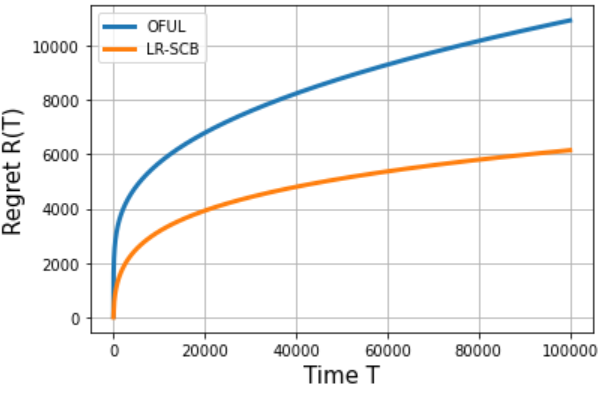}}
    \caption{Regret Scaling with respect to horizon $T$ for OFUL and \texttt{LR-SCB}. The plots are produced by taking an average over $50$ trials.}
    \label{fig:reg}
\end{figure*}

\section{Shifted OFUL}
\label{sec:shift-oful_main}
In this section, we establish a relationship between the regret of the standard OFUL algorithm and the shift OFUL for linear contextual bandits, and show that shifts can not reduce the regret of OFUL. We crucially leverage the analysis of shifted OFUL in Algorithm~\ref{algo:main_algo}. Beyond Algorithm~\ref{algo:main_algo}, this analysis may be of independent interest.

We keep the problem setup same as Section~\ref{sec:setup}. We define the shifted version of OFUL below.

Recall that the OFUL algortihm is used to make a decision of which action to take at time-step $t$, given the history of past actions $X_1, \cdots, X_{t-1}$ and observed rewards $Y_1, \cdots, Y_{t-1}$. The $\Gamma$ shifted OFUL is an algorithm identical to OFUL that describes the action to take at time step $t$, based on the past actions $X_1, \cdots, X_{t-1}$ and the observed rewards $\widetilde{Y}_1^{(\Gamma)}, \cdots, \widetilde{Y}_{t-1}^{(\Gamma)}$, where for all $1 \leq s \leq t-1$, $\widetilde{Y}_s = Y_s - \langle X_s, \Gamma \rangle$. 

Let us first recall the definition of regret for an un-shifted standard OFUL instance.
\begin{definition} [OFUL]
For a linear contextual bandit instance with unknown parameter $\theta^*$, and a sequence of (possibly random) actions $X_{1:T} := X_1, \cdots, X_T$, we denote the regret obtained upto round $T$ as
\begin{align*}
    R_T(X_{1:T}) := \sum_{t=1}^T  \max_{1 \leq j \leq K} \langle \beta_{j,t} - X_t, \theta^* \rangle.
\end{align*}
\end{definition}
Using the same notation as above, we now define the regret of an instance of the $\Gamma$ shifted system. 
\begin{definition} [$\Gamma$ shifted OFUL]
For a linear contextual bandit system with unknown parameter $\theta^*$, the modified set of rewards and a sequence of (possibly random) actions $X_{1:T}:=X_1, \cdots, X_T$, we denote its regret upto time $T$ as
\begin{align*}
    R_{T}^{(\Gamma)}(X_{1:T}) := \sum_{t=1}^T \max_{1 \leq j \leq K} \langle \beta_{j,t} - X_t, \theta^* - \Gamma \rangle
\end{align*}
\end{definition}
We now show that the shifted OFUL algorithm incurs higher regret than that of unshifted one, with high probability. We have the following result.
\begin{lemma}
\label{lem:shift-oful_main}
Consider a linear contextual bandit instance with parameter $\theta^*$ with $||\theta^*|| \leq  1$ and the context vectors at each time are sampled independently from any (coordinate-wise) bounded distribution (i.e., $[-c/\sqrt{d},c/\sqrt{d}]^{\otimes d}$) for a constant $c$. Let $\Gamma \in \mathbb{R}^d$ be such that $|| \theta^* - \Gamma || \leq \psi $ for a constant $\psi < \frac{1}{2\sqrt{2}}$, and $X_{1:T} = (X_1, \cdots, X_T)$ be the set of actions chosen by the $\Gamma$ shifted OFUL. Then, with probability at-least $\left( 1- {K \choose 2} e^{- c_1 d} - Ke^{- c_2 \,d} \right)$,
\begin{align*}
    \mathcal{R}_T(X_{1:T}) \leq \mathcal{R}^{(\Gamma)}_T(X_{1:T}),
\end{align*}
where the constants $c_1$ and $c_2$ depend on $\psi$.
\end{lemma}
\begin{remark}
The above lemma shows that for a deterministic $\Gamma$ shift, provided $d \geq \Omega(\log K)$, the shifted system always suffers higher regret with probability at least $1- c \exp(-c_1 \,d)$
\end{remark}
\subsubsection{Proof Sketch}
 The proof of the above Lemma is deferred in Appendix~\ref{sec:shift-oful}. We now give a brief sketch here. To show the above, we first show the following using definitions and some basic facts in optimization literature.
\begin{proposition}
\label{prop:shift}
Suppose for a linear contextual bandit instance with parameter $\theta^*$, an algorithm plays the sequence of actions $X_1, \cdots, X_T$, then
\begin{align*}
    \mathcal{R}_T(X_{1:T}) & \leq \mathcal{R}_{T}^{(\Gamma)}(X_{1:T}) \\
    & + \sum_{t=1}^T \left( \langle X_t -  \argmax_{\beta \in \{ \beta_{1,t}, \cdots, \beta_{K,t} \}} \langle \beta , \theta^* \rangle, \Gamma \rangle \right).
\end{align*}
\end{proposition}
From the above, it is clear that provided,
\begin{align*}
    \argmax_{\beta \in \{ \beta_{1,t}, \cdots, \beta_{K,t} \}} \langle \beta , \theta^* \rangle = \argmax_{\beta \in \{ \beta_{1,t}, \cdots, \beta_{K,t} \}} \langle \beta , \Gamma \rangle,
\end{align*}
the second term in Proposition~\ref{prop:shift} is negative, and we have Lemma~\ref{lem:shift-oful_main}. We now concentrate on the probability under which the above mentioned event occurs. For this, we use the anti-concentration property of the coordinate-wise bounded (and hence sub-Gaussian) random variables, along with the fact that the contexts are drawn in an independent manner. Leveraging these, we obtain the probability of the above-mentioned event is at least $1- {K \choose 2} e^{- c_1 d} - Ke^{- c_2 \,d}$, which proves the lemma.

%% file: simulations.tex
\section{Simulations}
\label{sec:sim}
\begin{figure*}[t!]
\label{fig:synthetic}
\centering
    \subfloat[][$d=20, K=20$]{
    \includegraphics[width=0.33\linewidth]{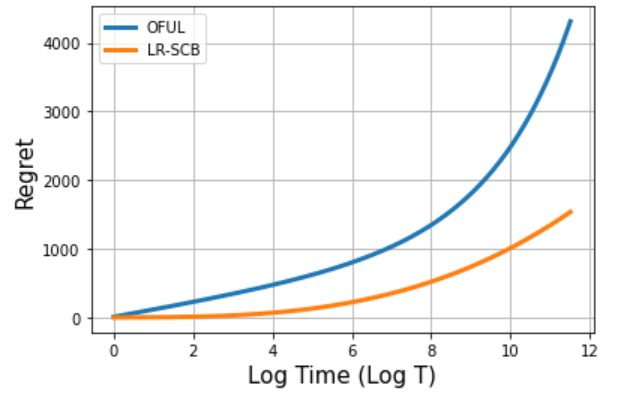}}
    \subfloat[][$d=25,K=20$]{
    \includegraphics[width=0.33\linewidth]{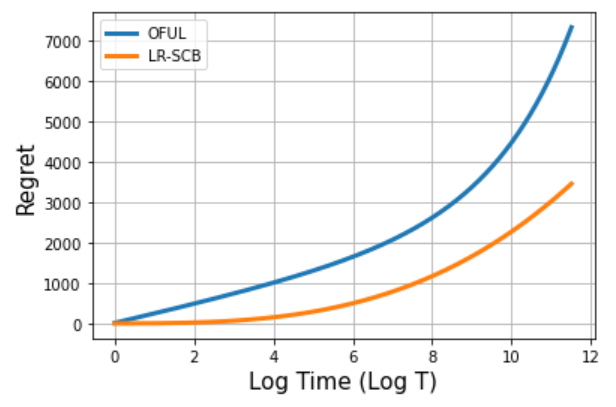}}
    \subfloat[][$d=30, K=20$]{
    \includegraphics[width=0.33\linewidth]{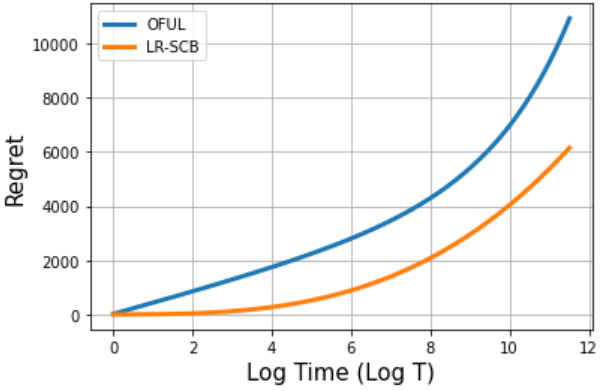}}
    \caption{Regret Scaling with respect to  $\log T$ for OFUL and \texttt{LR-SCB}. Note that the regret of \texttt{LR-SCB} grows much slowly, compared to OFUL. The plots are produced by taking an average over $50$ trials.}
    \label{fig:reg_one}
\end{figure*}
\begin{figure*}[t!]
\label{fig:synthetic}
\centering
    \subfloat[][$d=20, K=20$]{
    \includegraphics[width=0.33\linewidth]{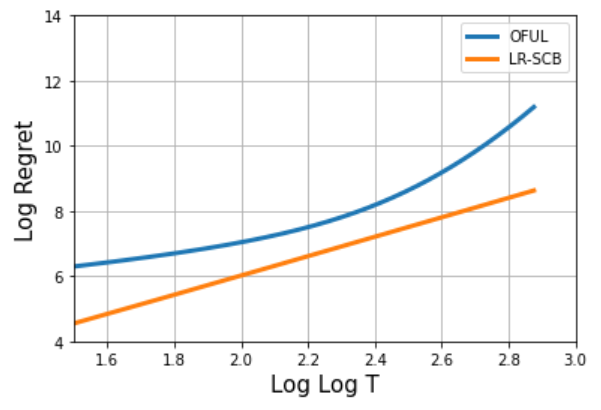}}
    \subfloat[][$d=25,K=20$]{
    \includegraphics[width=0.33\linewidth]{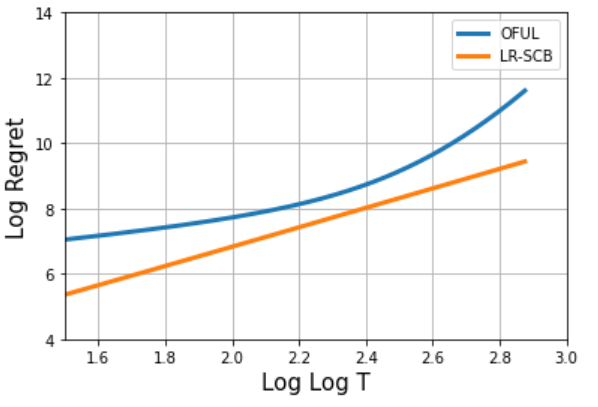}}
    \subfloat[][$d=30, K=20$]{
    \includegraphics[width=0.33\linewidth]{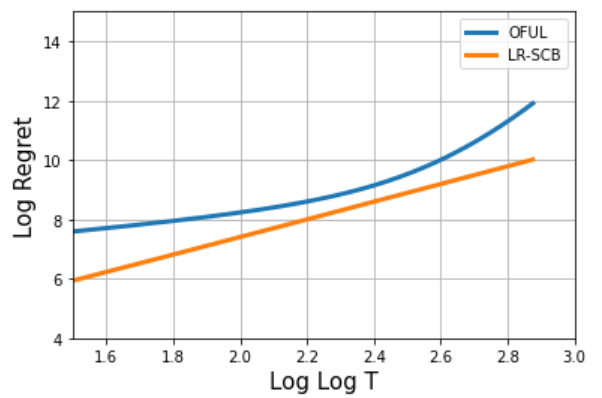}}
    \caption{ Scaling of $\log R(T)$ with respect to $\log \log T$ for OFUL and \texttt{LR-SCB}. The linear increase of \texttt{LR-SCB} indicates a $\polylog(T)$ regret. The plots are produced by taking an average over $50$ trials.}
    \label{fig:reg_two}
\end{figure*}

In this section, we validate our theoretical findings of Section~\ref{sec:result} via simulations. We assume that the contexts are drawn i.i.d from $\mathsf{Unif}[-1/\sqrt{d},1/\sqrt{d}]^{\otimes d}$. We run Algorithm~\ref{algo:main_algo} with $K = 20$ arms with different dimension $d = \{20,15,30\}$. Moreover, we compare our results with that of the OFUL (Algorithm~\ref{algo:oful}), and show the \texttt{LR-SCB} attanins much smaller regret compared to OFUL.

\subsubsection{$R(T)$ vs. $T$:}
We first plot the the variation of regret $R(T)$, with respect to the learning horizon $T$ for OFUL as well as \texttt{LR-SCB}, for different dimension $d \in \{20,25,30\}$. It is shown in Figure~\ref{fig:reg}. We observe that the regret of \texttt{LR-SCB} is much smaller than that of OFUL. This indeed validates our theoretical finding, since for OFUL, the regret $R_{OFUL}(T) = \mathcal{O}(\sqrt{T})$, and for \texttt{LR-SCB}, from Theorem~\ref{thm:main_theorem}, $R_{LR-SCB}(T) = \mathcal{O}(\polylog T)$. We run $50$ instances, and take average over trials to obtain the plots in Figure~\ref{fig:reg}.

\subsubsection{$R(T)$ vs. $\log T$}
To understand the regret scaling a bit better, we now plot the $R_{OFUL}(T) $ and $R_{LR-SCB}(T)$ with $\log T$. The plots are shown in Figure~\ref{fig:reg_one}. We observe here that the regret scales quite aggressively for OFUL, while it increases at a much slower rate for $\texttt{LR-SCB}$.

Note that since, $R_{OFUL}(T) = \mathcal{O}(\sqrt{T})$, the plot of $R_{OFUL}(T)$ vs. $\log T$ is expected to grow at an exponential speed, which we can see from  Figure~\ref{fig:reg_one} in all 3 cases. On the other hand, since $R_{LR-SCB}(T) = \mathcal{O}(\polylog T)$, the $R_{OFUL}(T)$ vs. $\log T$ plot is expected to grow at a polynomial rate, which is evidenced by the slow rate of increase. Hence, Figure~\ref{fig:reg_one} clearly hints towards a $\polylog(T)$ regret of \texttt{LR-SCB}, which validates Theorem~\ref{thm:main_theorem}.

\subsubsection{$\log R(T)$ vs. $\log \log T$}
In order to further understand the regret scaling of \texttt{LR-SCB}, we plot $\log R(T)$ against $\log \log T$, for both OFUL and $\texttt{LR-SCB}$. The results are shown in Figure~\ref{fig:reg_two}. Note that for \texttt{LR-SCB}, we obtain lines with slope slightly more than $2$. 

This clearly indicates a $\mathcal{O}(\polylog T)$ regret of $\texttt{LR-SCB}$.
Recall that the regret of $\texttt{LR-SCB}$ is  $R_{LR-SCB}(T) = \mathcal{O}(\polylog T)$, and hence $\log R_{LR-SCB}$ is a linear function of $\log \log T$, which we evidence. Furthermore, this hints that the polynomial dependence on $\log T$ is close to a quadratic one. On the other hand, for OFUL, note that the log regret is not a straight line, and keeps on increasing. This implies that the regret of OFUL is not poly-logarithmic, which matches the known results. We emphasize that, it is quite non-trivial to capture the regret of OFUL and \texttt{LR-SCB} in $\log \log T$ scale. Hence, we ran the learning algorithms for $T = 5\times 10^7$, to get the above mentioned results.

\section{Conclusion and Future work}

In this paper, we exploit the stochasticity of the contexts and obtain an instance-independent poly logarithmic regret bound for linear contextual bandits. Our analysis crucially relies on leveraging the norm adaptive learning algorithms, like \texttt{ALB-norm}. In this paper, we only obtain an upper bound, and hence a natural question arises about the tightness of the result. An immediate future work is to obtain an lower bound in the presence of stochastic context, and see whether our result is tight. Additionally, we want to understand the (structured) stochastic contextual bandit framework beyond linearity, and ask for similar guarantees. We keep these as our future endevors.

%% file: proofs.tex
\begin{center}
    \textbf{\Large{Supplementary Material for ``Logarithmic Regret for Stochastic Contextual Linear Bandits''}}
\end{center}
\vspace{3mm}
\section{Proof of Theorem~\ref{thm:main_theorem}}\
\label{app:main}
\paragraph{Regret in Phase 1:}
We run the OFUL algorithm (shown in Algorithm~\ref{algo:oful} for $T_1$ time steps. Hence, in this phase, the center indeed learns the parameter $\theta^*$. Let $\widehat{\theta}_{T_1}$ be the corresponding estimate. Provided, $T_1 > \tau_{\min}(\delta)$, from \citep{chatterji2020osom}, we have,
\begin{align*}
    \|\widehat{\theta}_{T_1} - \theta^*\| \leq \mathcal{O}\left( \sqrt{\frac{d}{\rho_{\min} T_1}} \right) \log(KT_1/\delta) \log(dT_1/\delta),
\end{align*}
with probability at least $1-\delta$. The corresponding regret (call it $R_{T_1}$) is
\begin{align*}
    R_{T_1} = \mathcal{O}\left( \sqrt{\frac{d T_1}{\rho_{\min} }} \right) \log(KT/\delta) \log(dT/\delta),
\end{align*}
with probability at least $1-\delta$.

\paragraph{Regret in Phase 2:}
In this phase, we take advantange of the learned paameter, $\widehat{\theta}_{T_1}$. Here, the learning proceeds as the following: 
At each time $t$, out of $K$ contexts, $\{\beta_{r,t}\}_{r=1}^K$, suppose the player chooses a context vector, $\beta_{r,t}$, (corresponding to the $r$-th arm). Thereafter, the player generates the reward $y_t = \langle \beta_{r,t},\theta^* \rangle + \xi_{i,t}$. Subsequently, using the previous estimate, the player calculates the corrected reward
\begin{align*}
    \Tilde{y}_t = y_t - \langle \beta_{r,t},\widehat{\theta}_{T_1} \rangle.
\end{align*}
Note that the player has the information about $(\beta_{r,t},\widehat{\theta}_{T_1})$ and so it can compute $\Tilde{y}_t$. With this shift, the center basically learns the vector $\theta^ - \widehat{\theta}_{T_1}$.

In this phase, we use a variation of the \texttt{ALB-norm} algorithm of \citep{ghosh_adaptive}\footnote{We reproduce the algorithm in Appendix~ \ref{sec:alb_norm}.}. The variation is reproduced in Section~\ref{sec:alb_norm}. Note that the \texttt{ALB-norm} algorithm is a norm adaptive algorithm, which is particularly useful when the parameter norm is small. \texttt{ALB-norm} uses the OFUL algorithm of \citep{chatterji2020osom} repeatedly over doubling epochs. At the beginning of each epoch, it estimates the parameter norm, and runs OFUL with the norm estimate (see \citep[Algorithm 1]{ghosh_adaptive}), and keeps on refining it. Hence, it is shown in \citep[Algorithm 1]{ghosh_adaptive} that while estimating the parameter $\Psi^*$, with high probability, the regret of \texttt{ALB-norm} is
\begin{align*}
    R_{\texttt{ALB-norm}} \leq \|\Psi^*\| \,\, R_{OFUL}.
\end{align*}

We use the \texttt{ALB-Norm} with this shifted system. However, since \texttt{ALB-Norm} is equivalent to playing the OFUL algorithm on doubling epochs, it is sufficient to obtain the performance of a shifted OFUL system, and the same conclusion extends to \texttt{ALB-Norm} (see \cite{ghosh_adaptive}). In Appendix~\ref{sec:shift-oful}, we present an analysis of shifted OFUL. In particular we show that shifts (by a fixed vector) can not reduce the regret (which is intuitive). Note that we learn $\widehat{\theta}_{T_1}$ in the previous phase, and fix it throughout  this phase. Hence, conditioned on the observations of the first phase, $\Hat{\theta}_{T_1}$ is a fixed (deterministic) vector. In particular, in Lemma~\ref{lem:shift-oful}, it is shown that provided $d \geq C \log (K^2 /\delta_2)$, we have $R_{OFUL} \leq R_{OFUL}^{shift}$ with probability at least $1-\delta_2$.

Hence, using Lemma~\ref{lem:shift-oful} of Appendix~\ref{sec:shift-oful}, the regret in phase 2 (call it $R_{T_2}$) is given by
\begin{align*}
    R_{T_2}  \leq \mathcal{O}\left( \|\theta^* - \Hat{\theta}_{T_1}\| \sqrt{\frac{d T_2}{\rho_{\min}}} \log(KT_2/\delta_2) \log(dT_2/\delta_2) \right),
\end{align*}
with probability at least $1- c\delta_2$, provided $d \geq C \log(K^2/\delta_2)$. Substituting, we obtain
\begin{align*}
   R_{T_2} \leq \mathcal{O}\left( \frac{d}{\rho_{\min}} \sqrt{\frac{T_2}{T_1}} \right) \log^2(KT_2/\delta_2) \log^2(dT_2/\delta_2)
\end{align*}
with probability exceeding $1-c\delta_2$.

\paragraph{Regret in Phase 3:} At the end of phase 2, we obtain the estimate $\widehat{\theta}_{T_2}$. Note that this is an estimate of $\theta^* - \widehat{\theta}_{T_1}$. In Phase 3, the idea is to exploit this estimate. The intuition is similar to that of phase 2. Since $\widehat{\theta}_{T_2}$ is an estimate of $\theta^* - \widehat{\theta}_{T_1}$, the quantity $\|\theta^* - \widehat{\theta}_{T_1} - \widehat{\theta}_{T_2}\|$ will be small, an a norm-adaptive algorithm, like \texttt{ALB-norm} should exploit this fact. 

In order to show this, we first show that, similar to the OFUL algorithm, it is possible for the \texttt{ALB-norm} algorithm to estimate the parameter of interest. In Appendix~\ref{sec:alb_norm_est}, we show this formally. Intuitively, this makes sense, since \texttt{ALB-norm} is basically the OFUL algorithm of \cite{chatterji2020osom} applied repeatedly over doubling epochs. Since, the OFUL algorithm estimates the underlying parameter, in Section~\ref{sec:alb_norm_est}, we show that \texttt{ALB-norm} also performs similar parameter estimation.

Furthermore, now the corrected regret is given by,
\begin{align*}
    \Tilde{y}_t = y_t - \langle \beta_{r,t}, \widehat{\theta}_{T_1} \rangle - \langle \beta_{r,t}, \widehat{\theta}_{T_2} \rangle.
\end{align*}

In other words, we shift the center by an amount given corresponding to $\widehat{\theta}_{T_1} + \widehat{\theta}_{T_2}$. We use the same analysis in Section~\ref{sec:shift-oful}  to show that provided $d \geq C \log(K^2/\delta_3)$, we have $R_{OFUL} \leq R_{OFUL}^{shift}$. Hence, the regret of this phase is given by,
\begin{align*}
     R_{T_3} &\leq \mathcal{O} \left( \|\theta^* - \Hat{\theta}_{T_1} - \widehat{\theta}_{T_2}\| \sqrt{\frac{d T_3}{\rho_{\min}}} \log(KT_3/\delta_3) \log(dT_3/\delta_3) \right) \\
   &\leq \mathcal{O}\left( \frac{d}{\rho_{\min}} \sqrt{\frac{T_3}{T_2}} \right) \log^2(KT_3/\delta_3) \log^2(dT_3/\delta_3)
\end{align*}
with probability at least $1-c\delta_3$.

\paragraph{Subsequent Phases:} For phase $i>3$, the same argument holds, and the regret is given by,
\begin{align*}
    R_{T_i} \leq \mathcal{O}\left( \frac{d}{\rho_{\min}} \sqrt{\frac{T_i}{T_{i-1}}} \right) \log^2(KT_i/\delta_i) \log^2(dT_i/\delta_i)
\end{align*}
with probability at least $1- c\delta_i$, provided $d \geq C \log (K^2/\delta_i)$.

\paragraph{Total Regret:} 
We now characterize the total regret of the agent. Let us assume the number of phases is $N$. We have
\begin{align*}
  R_T &= R_{T_1} + \ldots + R_{T_N} \\
    & \leq \mathcal{O}(\sqrt{\frac{d}{\rho_{\min}}}\sqrt{T_1}) \log(KT_1/\delta)\log(dT_1/\delta) + \sum_{i=2}^N \mathcal{O}\left( \frac{d}{\rho_{\min}} \sqrt{\frac{T_i}{T_{i-1}}} \right) \log^2(KT_i/\delta_i) \log^2(dT_i/\delta_i).
\end{align*}
Since we consider $\delta_i = \delta/2^{i-1}$, the above regret holds with probability at least
\begin{align*}
    & 1 - c (\delta_1 + \delta_2 + \ldots + \delta_N) \\
    &\geq 1 - (\delta + \delta/2 + \delta/4 + \ldots) \\
   & \geq 1 - 2 c\delta,
\end{align*}
where $c$ is an universal constant.

We now choose the length of phases as
\begin{align*}
    T_i = T_1 (\log T)^{i-1},
\end{align*}
where $T_1$ is the initial length. With this, we obtain, the number of epochs, $N = \mathcal{O}\left( \frac{\log (T/T_1)}{\log \log T} \right)$. Subsequently, the overall regret is given by,
\begin{align*}
    R_T \leq  \mathcal{O} \left [(\sqrt{\frac{d}{\rho_{\min}}}\sqrt{T_1}) \log(KT/\delta)\log(dT/\delta) + N \frac{d}{\rho_{\min}} \sqrt{\log T} \,\, N^2 \log^2(KT_1 (\log T)/\delta) \,\, N^2 \log^2(dT_1 (\log T)/\delta) \right],
\end{align*}
where we substitute $\delta_i$ and upper bound the number of epochs by $N$. Substituting $N$, we obtain
\begin{align*}
    R_T &\leq  \mathcal{O} \left [(\sqrt{\frac{d}{\rho_{\min}}}\sqrt{T_1}) \log(KT_1/\delta)\log(dT_1/\delta) +  \frac{d\sqrt{\log T}}{\rho_{\min}}  \left( \frac{\log (T/T_1)}{\log \log T} \right)^5 \log^2(KT_1 (\log T)/\delta) \log^2(dT_1 (\log T)/\delta) \right],
\end{align*}
with probability at least $1-c\delta$, provided
\begin{align*}
    d \geq C N \,\, \log (K^2/\delta) \geq C \,\, \left( \frac{\log (T/T_1)}{\log \log T} \right) \log (K^2/\delta)
\end{align*}

The next job is to choose the length of the first epoch $T_1$. For the norm adaptive algorithm, \texttt{ALB-norm} to work, one needs (from \citep[Theorem 1]{ghosh_adaptive})
\begin{align*}
    T_1 = C \max \bigg \lbrace \frac{d^2}{\rho_{\min}^2} \log^4(KT/\delta) , \tau_{\min}(\delta)^2 \bigg \rbrace
\end{align*}
for a large enough universal constant $C$, where $\tau_{\min} = \bigg[\frac{16}{\rho_{\min}^2}+\frac{8}{3\rho_{\min}} \bigg] \log(\frac{2dT}{\delta})$.
Hence, we need to choose
\begin{align*}
    T_1 = C_1 \, \frac{d^2}{\rho_{\min}^2} \log^4(KT/\delta) \log(dT/\delta).
\end{align*}

To ease notation, let us define
  \begin{align*}
    \mathsf{\Lambda} = \left( \frac{1}{(\log \log T)} \log \left(\frac{\rho_{\min}^2 T}{d^2 \log^4(KT/\delta) \log(dT/\delta)}\right) \right)
\end{align*}  
and,
\begin{align*}
    \mathfrak{T} = \log^3 \left(  \frac{K d^2 (\log T) (\log^4 KT/\delta) (\log dT/\delta)}{\rho_{\min}^2 \, \delta} \right) \log^2 \left(  \frac{d^3 (\log T) (\log^4 KT/\delta) (\log dT/\delta)}{\rho_{\min}^2 \, \delta} \right)
\end{align*}
With this, the overall regret is given by
\begin{align*}
    R_T &\leq  \mathcal{O} \left [\left( \frac{d}{\rho_{\min}} \right)^{3/2} \mathfrak{T} + \left( \frac{d}{\rho_{\min}} \right) \,\mathsf{\Lambda}^5 \, \mathfrak{T} \, \sqrt{\log T} \right] \\
    & \leq \mathcal{O} \left[ \left( \frac{d}{\rho_{\min}} \right)^{3/2} \,\mathsf{\Lambda}^5 \,\, \mathfrak{T} \,\, \sqrt{\log T} \right],
\end{align*}
with probability at least $1-c\delta$. This requires,
\begin{align*}
    d \geq C \,\, \left( \frac{\log (T/T_1)}{\log \log T} \right) \log (K^2/\delta).
\end{align*}
Since, $T_1$ is a function of $d$, we choose a sufficient condition on $d$, which is given by
\begin{align*}
    d \geq C \left(\frac{\log T}{\log \log T} \right) \log (K^2/\delta),
\end{align*}
which concludes the proof.

\section{ Modified {\ttfamily ALB-Norm} from \citep{ghosh_adaptive}}
\label{sec:alb_norm}
In this section, we reproduce {\ttfamily ALB-Norm} from \citep{ghosh_adaptive}, and prove a Corollary of the main theorem from \citep{ghosh_adaptive}.


\begin{corollary}[Corollary of Theorem $1$ from \citep{ghosh_adaptive}]
The regret of Algorithm \ref{algo:alb_norm} at the end of $T$ time-steps satisfies with probability at-least $1- 18\delta_1 - \delta_s$, 
\begin{align*}
    R(T) \leq C\|\theta^*\|(\sqrt{K}+\sqrt{d})\sqrt{T}\log\left(\frac{KT}{\delta_1}\right),
\end{align*}
where $C$ is an universal constant.
\label{cor:alb_norm_improved}
\end{corollary}

The proof follows by recomputing Lemma $1$ from \citep{ghosh_adaptive} as follows.
\begin{lemma}
If $T$ is sufficiently large such that $\frac{2C\sigma\sqrt{d}}{T^{\frac{1}{4}}}\log \left( \frac{K \sqrt{T}}{\delta_1}\right) \leq 1$, then with probability at-least $1-8\delta_1 - \delta_s$, for all $i$ large, $b_i \leq 2 \|\theta^* \|$ holds, where $b_i$ is defined in Line $11$ of Algorithm \ref{algo:alb_norm}.
\label{lem:recompute}
\end{lemma}
\begin{proof}[Proof of Lemma \ref{lem:recompute}]
We start with Equation $(8)$ of \citep{ghosh_adaptive}. Reproducing Equation $(8)$ by substituting $T_1 = \lceil \sqrt{T} \rceil$, with probability at-least $1-8\delta_1$, for all phases $i \geq 2$,
\begin{align}
   b_{i+1} \leq \|\theta^*\| + ip\frac{b_i}{2^{\frac{i-1}{2}}T^{\frac{1}{4}}} + iq\frac{\sqrt{d}}{2^{\frac{i-1}{2}}T^{\frac{1}{4}}},
   \label{eqn:base_recursion}
\end{align}
holds, where $p$ and $q$ are defined in \citep{ghosh_adaptive} as 
\begin{align*}
    p &= \left( \frac{14\log \left( \frac{2K\sqrt{T}}{\delta_1} \right)}{\sqrt{\rho_{min}}} \right),\\
    q &= \left( \frac{2C\sigma \log \left( \frac{2K\sqrt{T}}{\delta_1} \right)}{\sqrt{\rho_{min}}} \right).
\end{align*}

For all $i \geq 2$, $\frac{i}{2^{\frac{i-1}{2}}} \leq 2$. Thus, for all $i \geq 1$, Equation (\ref{eqn:base_recursion}) can be rewritten as 
\begin{align}
    b_{i+1} &\leq \|\theta^* \| + \frac{pb_i}{T^{\frac{1}{4}}} + \frac{q \sqrt{d}}{T^{\frac{1}{4}}},\nonumber \\
    &\leq \|\theta^*\| + \frac{C\sigma\sqrt{d}}{T^{\frac{1}{4}}}\log \left( \frac{K \sqrt{T}}{\delta_1}\right)b_i.
       \label{eqn:base_recursion_2}
\end{align}
where $b_1 := 1$. We set this initial estimate as $1$, since $\max_{i \in \{1,\cdots, N\}}\|\theta^*_i\| \leq 1$. We prove the lemma by induction that $b_i \leq 2\|\theta^*\|$. 
\\

\textbf{Base case, $i=1$} - We know from the initialization (Line $3$ of Algorithm \ref{algo:alb_norm}), that with probability at-least $1-\delta_s$,
\begin{align*}
    b_1 &\leq \|\theta^*\| + \sqrt{2}\sigma \sqrt{\frac{d}{\tau}\log \left( \frac{1}{\delta_s}\right)}, \\
    &\leq 2 \|\theta^*\|.
\end{align*}
where $\tau$ and $\delta_s$ are defined in Line $2$ and input respectively of Algorithm \ref{algo:alb_norm}. 
\\

\textbf{Induction Step} - Assume that for some $i \geq 1$, for all $1 \leq j \leq i$, $b_j \leq 2 \|\theta^*\|$. Now, consider case $i+1$. From recursion in Equation (\ref{eqn:base_recursion_2}), that
\begin{align*}
    b_{i+1} &\leq \|\theta^* \| +\frac{C\sigma\sqrt{d}}{T^{\frac{1}{4}}}\log \left( \frac{K \sqrt{T}}{\delta_1}\right)b_i,\\
    &\stackrel{(a)}{\leq} \|\theta^*\|\left( 1 + \frac{2C\sigma\sqrt{d}}{T^{\frac{1}{4}}}\log \left( \frac{K \sqrt{T}}{\delta_1}\right) \right), \\
    &\stackrel{(b)}{\leq} 2 \|\theta^*\|.
\end{align*}
Step $(a)$ follows from the induction hypothesis. Step $(b)$ follows from the fact that $T$ is large enough such that $\frac{2C\sigma\sqrt{d}}{T^{\frac{1}{4}}}\log \left( \frac{K \sqrt{T}}{\delta_1}\right) \leq 1$. This concludes the proof of Lemma.
\end{proof}

\section{Parameter estimation for modified \texttt{ALB-norm} }
\label{sec:alb_norm_est}

In this section we show that, similar to the OFUL algorithm of \cite{chatterji2020osom}, the modified \texttt{ALB-norm} algorithm described in the previous section, also estimated the underlying parameter while minimizing regret. We have the following result:
\begin{proposition}
Suppose we run the modified \texttt{ALB-norm} algorithm, with underlying parameter $\Psi$ for $\mathcal{T}$ rounds (with the same stochastic context assumptions given in Section~\ref{sec:setup}. The estimate returned by \texttt{ALB-norm} satisfies
\begin{align*}
    \|\widehat{\Psi} - \Psi \| \leq \mathcal{O}\left( \sqrt{\frac{d}{\rho_{\min}\mathcal{T} }} \right) \log(K \mathcal{T}/ \delta) \log(d \mathcal{T}/\delta)),
\end{align*}
with probability at least $1-\delta$.
\end{proposition}

\begin{proof}
As shown in Algorithm~\ref{algo:alb_norm}, the \texttt{ALB-norm}
, algorithm works in doubling epochs. At each epoch, it runs the OFUL algorithm of \cite{chatterji2020osom} with a modified norm estimate. Let the doubling epochs be defined as $\{\mathcal{T}_1,\ldots,\mathcal{T}_N$, where $N$ is the total number of epochs. Also, the parameter-estimate at the end of the last epoch is $
\widehat{\Psi}$. Since, \texttt{ALB-norm} plays OFUL at the last epoch, we obtain,
\begin{align*}
    \|\widehat{\Psi} - \Psi \| \leq \mathcal{O}\left( \sqrt{\frac{d}{\rho_{\min}\mathcal{T}_N }} \right) \log(K \mathcal{T}_N/\delta) \log(d \mathcal{T}_N/\delta))
\end{align*}
with probability at least $1-\delta$. Now we have $\mathcal{T}_N \leq \mathcal{T}$ and,
\begin{align*}
    \mathcal{T}_N + \mathcal{T}_{N-1} + \ldots + \mathcal{T}_1 = \mathcal{T}.
\end{align*}
With the doubling epochs, we have
\begin{align*}
   & \mathcal{T}_N + \mathcal{T}_{N}/2 + \ldots  \geq \mathcal{T} \\
    & \mathcal{T}_N \left( 1+ 1/2 + \ldots \right) \geq \mathcal{T}\\
    & \mathcal{T}_N \geq \mathcal{T}/2.
\end{align*}
Substituting the above, we have
\begin{align*}
    \|\widehat{\Psi} - \Psi \| \leq \mathcal{O}\left( \sqrt{\frac{d}{\rho_{\min}\mathcal{T}}} \right) \log(K \mathcal{T}/\delta) \log(d \mathcal{T}/\delta))
\end{align*}
with probability at least $1-\delta$, which concludes the proof.
\end{proof}

\section{Shifted OFUL Regret}
\label{sec:shift-oful}

Here, we establish a relationship between the regret of the standard OFUL algorithm and the shift compensated algorithm. We define the shifted version of OFUL below.

\begin{definition}
The OFUL algortihm is used to make a decision of which action to take at time-step $t$, given the history of past actions $X_1, \cdots, X_{t-1}$ and observed rewards $Y_1, \cdots, Y_{t-1}$. The $\Gamma$ shifted OFUL is an algorithm identical to OFUL that describes the action to take at time step $t$, based on the past actions $X_1, \cdots, X_{t-1}$ and the observed rewards $\widetilde{Y}_1^{(\Gamma)}, \cdots, \widetilde{Y}_{t-1}^{(\Gamma)}$, where for all $1 \leq s \leq t-1$, $\widetilde{Y}_s = Y_s - \langle X_s, \Gamma \rangle$. 
\end{definition}

\begin{definition}
For a linear bandit instance with unknown parameter $\theta^*$, and a sequence of (possibly random) actions $X_{1:T} := X_1, \cdots, X_T$, denote by $\mathcal{R}_T(X_{1:T}) := \sum_{t=1}^T  \max_{1 \leq j \leq K} \langle \beta_{j,t} - X_t, \theta^* \rangle$.
\end{definition}

\begin{definition}
For a linear bandit system with unknown parameter $\theta^*$, and a sequence of (possibly random) actions $X_{1:T}:=X_1, \cdots, X_T$, denote by $\mathcal{R}_{T}^{(\Gamma)}(X_{1:T}) := \sum_{t=1}^T \max_{1 \leq j \leq K} \langle \beta_{j,t} - X_t, \theta^* - \Gamma \rangle$.
\end{definition}

\begin{proposition}
Suppose for a linear bandit instance with parameter $\theta^*$, an algorithm plays the sequence of actions $X_1, \cdots, X_T$, then
\begin{align*}
    \mathcal{R}_T(X_{1:T}) \leq \mathcal{R}_{T}^{(\Gamma)}(X_{1:T}) + \sum_{t=1}^T \left( \langle X_t -  \argmax_{\beta \in \{ \beta_{1,t}, \cdots, \beta_{K,t} \}} \langle \beta , \theta^* \rangle, \Gamma \rangle \right).
\end{align*}
\label{prop:shift_regret}
\end{proposition}
\begin{proof}
From the definition of $\mathcal{R}_{T}^{(\Gamma)}$, we can write the regret as 
\begin{align}
    \mathcal{R}_{T}^{(\Gamma)}(X_{1:T}) &=  \sum_{t=1}^T  \max_{1 \leq j \leq K} \langle \beta_{j,t} - X_t, \theta^* + \Gamma \rangle, \nonumber \\
    & \stackrel{(a)}{\leq} \sum_{t=1}^T \max_{1 \leq j \leq K} \langle \beta_{j,t} , \theta^* \rangle +  \langle \beta^*_t, \Gamma \rangle - \langle X_t, \theta^* \rangle - \langle X_t, \Gamma \rangle, \label{eqn:regret_shift_decomp}
\end{align}
where, $\beta^*_t := \argmax_{\beta \in \{\beta_{1,t}, \cdots, \beta_{K,t} \}} \langle \beta, \theta^* \rangle$. The inequality $(a)$ follows from the following elementary fact.
\begin{lemma}
Let $\mathcal{X}$ be a compact set, and functions $f,g : \mathcal{X} \rightarrow \mathbb{R}$, such that $\sup_{x \in \mathcal{X}} |f(x)| < \infty$ and $\sup_{x \in \mathcal{X}}|g(x)| < \infty$. Then, 
$$\max_{x \in \mathcal{X}}(f(x) + g(x)) \geq \max_{x \in \mathcal{X}} f(x) + \min_{x \in \mathcal{X}} g(x).$$
\end{lemma}
that Rewriting Equation (\ref{eqn:regret_shift_decomp}), we see that 
\begin{align*}
    \mathcal{R}_{T}^{(\Gamma)}(X_{1:T}) \leq \mathcal{R}_T + \sum_{t=1}^T \langle \beta^*_t - X_t, \Gamma \rangle,
\end{align*}
and thus the proposition is proved.
\end{proof}


\begin{corollary}
Suppose for all time $t$, $\argmax_{\beta \in \{ \beta_{1,t}, \cdots, \beta_{K,t} \}} \langle \beta , \theta^* \rangle = \argmax_{\beta \in \{ \beta_{1,t}, \cdots, \beta_{K,t} \}} \langle \beta , \Gamma \rangle$. Then, 
\begin{align*}
    \mathcal{R}_T(X_{1:T}) \leq \mathcal{R}_{T}^{(\Gamma)}(X_{1:T}).
\end{align*}
\end{corollary}
\begin{proof}
From the hypothesis of the theorem, we can observe the following, 
\begin{align*}
    \sum_{t=1}^T \left( \langle X_t -  \argmax_{\beta \in \{ \beta_{1,t}, \cdots, \beta_{K,t} \}} \langle \beta , \theta^* \rangle, \Gamma \rangle \right) &= \sum_{t=1}^T \left( \langle X_t -  \argmax_{\beta \in \{ \beta_{1,t}, \cdots, \beta_{K,t} \}} \langle \beta , \Gamma \rangle, \Gamma \rangle \right), \\
    &\leq 0.
\end{align*}
Plugging the above bound into Proposition \ref{prop:shift_regret} completes the proof. 
\end{proof}

\subsubsection{ High Probability Bound on $\mathcal{R}_T^{(\Gamma)}$ }

\begin{lemma}
Suppose the $K$ context vectors $\beta_1, \cdots, \beta_K$ are such that for all $i$, $||\beta_i|| \leq 2$ and for all $i \neq j$, $ | \langle \beta_i - \beta_j , \theta^* \rangle |  \geq 4  || \theta^* - \Gamma||$, where $\theta^*$ is the unknown linear bandit parameter and $\Gamma$ is a fixed vector. Then
\begin{align*}
    \argmax_{1\leq j \leq K} \langle \beta_j, {\theta}^*_i \rangle = \argmax_{1\leq j \leq K} \langle \beta_j, \Gamma \rangle.
\end{align*}
\label{lem:when_are_optimal_equal}
\end{lemma}
\begin{proof}
We will prove the following more stronger statement. Let $i \neq j \in [K]$ be such that $\langle \theta^*, \beta_i \rangle \geq \langle \theta^*, \beta_j \rangle$. Then, under the hypothesis of the proposition statement, we have $\langle \theta^*, \beta_i - \beta_j \rangle \geq 4 || \theta^* - \Gamma ||$. Thus, the following chain holds, 
\begin{align*}
    \langle \beta_i - \beta_j, \Gamma \rangle &= \langle  \beta_i - \beta_j, \theta^*\rangle + \langle  \beta_i - \beta_j, \Gamma - \theta^* \rangle, \\
    &\geq 4 || \theta^* - \Gamma || + \langle  \beta_i - \beta_j, \Gamma - \theta^* \rangle, \\
    &\geq 4  || \theta^* - \Gamma || - ||\beta_i-\beta_j||||\Gamma - \theta^*||, \\
    &\geq 0.
\end{align*}
The first inequality follows from the hypothesis of the proposition statement, the second follows from Cauchy Schwartz inequality and the last follows from the fact that $||\beta_i - \beta_j|| \leq 2$. Thus, we have shown that under the hypothesis of the Proposition, the ordering of the coordinates whether by inner product with $\theta^*$ or with $\Gamma$ remains unchanged. In particular, the argmax is identical.
\end{proof}

\begin{lemma}
Let $\theta^*$ be a fixed vector with $\|\theta^*\| \leq 1$, and $\Gamma \in \mathbb{R}^d$ be any arbitrary vector such that $|| \theta^* - \Gamma|| \leq \psi$, for some constant $\psi$. Let $\beta_1, \cdots, \beta_K$ be i.i.d. vectors, supported on $[-c/\sqrt{d},c/\sqrt{d}]^{\otimes d}$ for a constant $c$. Then, 
\begin{align*}
    \mathbb{P} \left[     \argmax_{1\leq j \leq K} \langle \beta_j, {\theta}^*_i \rangle = \argmax_{1\leq j \leq K} \langle \beta_j, \Gamma \rangle \right] \geq \left( 1- {K \choose 2} e^{-\frac{d}{4}(1-8\psi^2)^2} - K e^{-\frac{\sqrt{5}-1}{2}d} \right).
    \end{align*}
\label{lem:anti_concentration}
\end{lemma}

\begin{proof}
Denote by the \emph{Good event} $\mathcal{E} := \left\{ \argmax_{1\leq j \leq K} \langle \beta_j, {\theta}^*_i \rangle = \argmax_{1\leq j \leq K} \langle \beta_j, \Gamma \rangle \right\}$
From Lemma \ref{lem:when_are_optimal_equal}, we know that a sufficient condition for event $\mathcal{E}$ to hold is that for all $i \neq j$, we have $\bigg|\langle \theta^*, \beta_i - \beta_j \rangle\bigg| \geq 2 || \theta^* - \Gamma ||$ and for all $i$, $||\beta_i|| < 2$. Thus, from a simple union bound, we get
\begin{align*}
    \mathbb{P}[\mathcal{E}^c] &\leq \sum_{1 \leq i < j \leq K} \mathbb{P} \left[\bigg|\langle \theta^*, \beta_i - \beta_j \rangle\bigg| \leq 4  || \theta^* - \Gamma ||\right] + \sum_{i=1}^K \mathbb{P}[||\beta_i|| \geq 2], \\
    & = {K \choose 2} \mathbb{P} \left[\bigg|\langle \theta^*, \beta_1 - \beta_2 \rangle\bigg| \leq 4 || \theta^* - \Gamma ||\right] + K \mathbb{P}[||\beta_1|| \geq 2].
\end{align*}
The second equality follows from the fact that $\beta_1, \cdots, \beta_K$ are i.i.d. Now, since $||\theta^*|| \leq 1$, we have from Cauchy Schwartz that, almost-surely, $\bigg|\langle \theta^*, \beta_1 - \beta_2 \rangle\bigg|  \leq || \beta_1 - \beta_2 ||$. Thus, 
\begin{align*}
    \mathbb{P} \left[|\langle \theta^*, \beta_1 - \beta_2 \rangle | \leq 4 || \theta^* - \Gamma || \right] &\leq \mathbb{P} [ || \beta_1 - \beta_2 ||  \leq 4  || \theta^* - \Gamma ||], \\
    & \leq  \mathbb{P} [ || \beta_1 - \beta_2 || \leq 4 \psi ], \\
    &= \mathbb{P} [ || \beta_1 - \beta_2 ||^2 \leq 16 \psi^2 ], \\
    &\stackrel{(a)}{\leq} e^{-\frac{c_1 d}{4}},
\end{align*}
where the constant $c_1$ depends on $\psi$. The first inequality follows from Cauchy Schwartz, and the fact that $|| \theta^* || \leq 1$. The last inequality follows from the fact that, $\mathbb{E}\| \beta_1 - \beta_2\|^2 = c_2$ for a constant $c_2$, and since $\{\beta_1,\beta_2\}$ are coordinate-wise bounded, we use standard sub-Gaussian concentration to argue that $\|\beta_1 - \beta_2\|^2$ is close to its expectation. Finally, we obtain that 
\begin{align*}
    \mathbb{P}\left( \|\beta_1 - \beta_2\|^2 - \mathbb{E}\|\beta_1 - \beta_2\|^2 \leq -t \right) \leq \exp \left ( - c_3 \,d t^2  \right).
\end{align*}
Choosing $t$ as a constant, we obtain (a).

Finally, we also need to ensure that the context vectors $\beta_1, \cdots, \beta_K$ have norms bounded by $2$. This can also be similarly be bounded by the upper tail inequality as
    \begin{align*}
        \mathbb{P} [ || \beta_1|| \geq 2 ] &= \mathbb{P}[ d || \beta_1||^2 \geq 4 d], \\
        &\stackrel{(b)}{\leq} e^{- c_4 d}. 
    \end{align*}
for a constant $c_4$, where inequality $(b)$ follows from the upper-tail concentration bound for sub-Gaussian random variables. Putting this all together concludes the proof.
\end{proof}

\begin{lemma}
\label{lem:shift-oful}
Consider a linear bandit instance with parameter $\theta^*$ with $||\theta^*|| \leq  1$ and the context vectors at each time are sampled uniformly and independently from on a distribution with support  $[-c/\sqrt{d},c/\sqrt{d}]^{\otimes d}$ for a constant $c$, i.e., the contexts are i.i.d. across time and arms. Let $\Gamma \in \mathbb{R}^d$ be such that $|| \theta^* - \Gamma || \leq \psi $ for a constant $\psi < \frac{1}{2\sqrt{2}}$, and $X_{1:T} = (X_1, \cdots, X_T)$ be the set of actions chosen by the $\Gamma$ shifted OFUL. Then, with probability at-least $\left( 1- {K \choose 2} e^{- c_1 d} - Ke^{- c_2 \,d} \right)$,
\begin{align*}
    \mathcal{R}_T(X_{1:T}) \leq \mathcal{R}^{(\Gamma)}_T(X_{1:T}),
\end{align*}
where the constants $c_1$ and $c_2$ depend on $\psi$.
\label{Lem:shift_equal_true_regret_whp}
\end{lemma}
\begin{proof}
This follows by combining Lemma \ref{lem:anti_concentration} and \ref{lem:when_are_optimal_equal}.
\end{proof}
\vfill